\documentclass[10pt,letterpaper]{article}

\usepackage{cogsci}

\cogscifinalcopy 

\usepackage[
  style=apa,
  natbib=true,
  annotation=false,
]{biblatex}
\usepackage{float}
\usepackage{graphicx}
\usepackage{booktabs}
\usepackage{amsmath}

\usepackage{amssymb}
\usepackage{multirow}
\usepackage{xcolor}
\usepackage[hidelinks]{hyperref}

\definecolor{cbblue}{RGB}{0,114,178}      
\definecolor{cborange}{RGB}{213,94,0}     
\definecolor{cbsky}{RGB}{86,180,233}      
\definecolor{cbyellow}{RGB}{240,228,66}   

\usepackage{tikz}
\usepackage{pgfplots}
\pgfplotsset{compat=1.17}
\usetikzlibrary{patterns,arrows.meta,positioning,calc}

\newtheorem{theorem}{Theorem}
\newtheorem{proposition}{Proposition}
\newtheorem{corollary}{Corollary}
\newtheorem{definition}{Definition}

\newenvironment{proof}[1][Proof]{\noindent\textit{#1.} }{\hfill$\square$\par\medskip}

\title{When Can Human-AI Teams Outperform Individuals? Tight Bounds with Impossibility Guarantees}

\author[1]{\mbox{Dongxin Guo (bettyguo@connect.hku.hk)}}
\author[2,3]{\mbox{Jikun Wu}}
\author[1]{\mbox{Siu-Ming Yiu}}
\affil[1]{The University of Hong Kong, Hong Kong, China}
\affil[2]{Stellaris AI Limited, Hong Kong, China}
\affil[3]{Brain Investing Limited, Hong Kong, China}

\begin{document}

\maketitle

\begin{abstract}
	Human-AI teams fail to outperform their best member in 70\% of studies, yet no theory specifies when complementarity is achievable. We derive tight bounds for the broad class of \emph{confidence-based aggregation} rules by integrating signal detection theory with information-theoretic analysis, yielding four results: (1) a complementarity theorem (teams outperform individuals iff error correlation $\rho_{HM} < \rho^*$, with $\rho^* \approx a$ in the symmetric near-chance regime); (2) minimax bounds showing gains scale as $\Theta(\sqrt{\Delta d})$ with metacognitive sensitivity difference; (3) an impossibility result proving no confidence-based aggregation rule achieves complementarity when $\rho_{HM} \geq \rho^*$; and (4) multi-class generalization $\rho^*_K \approx \rho^*/\sqrt{K-1}$. Predictions match observed team accuracy ($R = 0.94$ on ImageNet-16H, $R = 0.91$ on CIFAR-10H) and the multi-class threshold scaling holds on human data ($R = 0.93$, $K = 16$), with robustness under non-Gaussian distributions. The framework explains why complementarity is rare and provides actionable design formulas; results apply to aggregation, not to interactive deliberation that generates novel answers.
	
\end{abstract}

\section{Introduction}

Artificial intelligence increasingly assists human decision-making in medical diagnosis \citep{Rajpurkar2022}, judicial decisions \citep{Kleinberg2018}, and autonomous systems \citep{Shalev2017}. The promise of human-AI collaboration rests on \textit{complementarity} (combined performance exceeding individual performance) \citep{Hemmer2024}. However, a meta-analysis of 106 studies found human-AI combinations performed \textit{worse} than the best individual on average (Hedges' $g = -0.23$), achieving complementarity in only 30\% of cases \citep{Vaccaro2024}.

This raises a fundamental question: \textit{Under what formal conditions does human-AI collaboration provably outperform either agent alone?} Despite research on trust calibration \citep{Zhang2020}, explanation quality \citep{Bansal2021}, and mental models \citep{Bansal2019}, no prior work has established tight bounds with matching achievability and impossibility results.

We address this gap by developing a mathematical framework unifying signal detection theory \citep{Maniscalco2012}, Bayesian aggregation \citep{Steyvers2022}, and ensemble learning \citep{Wood2023}.

\textbf{Scope.} Our results characterize \emph{confidence-based aggregation}: any rule $\mathcal{A}(\hat{Y}_H, \hat{Y}_M, c_H, c_M)$ that produces a team prediction from agents' answers and confidence signals. This class subsumes confidence-weighted selection, Bayesian model averaging, learned deferral policies, and ensemble methods. It does \emph{not} cover interactive deliberation in which agents generate new answers through dialogue \citep{Navajas2018, BarreraLemarchand2025} or reasoning processes that combine knowledge in ways absent from individual responses \citep{Page2007}; complementarity gains beyond our bounds may be achievable in such settings, but lie outside the present framework.

Our contributions are:

\begin{enumerate}
\item A \textbf{complementarity theorem} with explicit closed-form conditions specifying when team accuracy can exceed individual accuracy (Theorem~\ref{thm:complementarity});
\item \textbf{Tight minimax bounds} showing how metacognitive sensitivity (the ability to distinguish correct from incorrect predictions \citep{Fleming2014}) determines achievable gains (Theorem~\ref{thm:minimax});
\item An \textbf{impossibility result} proving that above a critical error correlation $\rho^*$, no aggregation achieves complementarity (Theorem~\ref{thm:impossibility});
\item \textbf{Multi-class generalization} validated on human behavioral data (ImageNet-16H, $K=16$), with robustness analysis under non-Gaussian assumptions.
\end{enumerate}

Our framework connects to classic results in collective intelligence. The Condorcet Jury Theorem \citep{Condorcet1785} shows that majority voting achieves perfect accuracy with independent voters; our bounds characterize exactly how independence violations (error correlation) constrain achievable gains. The ``wisdom of crowds'' phenomenon \citep{Surowiecki2004, Kameda2022} depends on diversity of opinions; our $\rho^*$ threshold quantifies precisely how much diversity (low error correlation) is required for complementarity to emerge.

We validate predictions on ImageNet-16H \citep{Steyvers2022} and CIFAR-10H, achieving $R=0.94$ and $R=0.91$ correlations between predicted and observed team accuracy. These results formalize when and why human-AI collaboration succeeds, and when it cannot.

\section{Related Work}

\subsection{Human-AI Complementarity}

\citet{Bansal2021} showed AI explanations do not reliably improve team performance when accuracies match. \citet{Inkpen2022} found expertise critically moderates complementarity. \citet{Hemmer2024} formalized Complementary Team Performance (CTP), identifying information and capability asymmetry as sources. \citet{Bahrami2010} demonstrated optimal cue combination in human dyads, providing foundations for multi-agent integration.

\subsection{Collective Intelligence}

Our work builds on collective intelligence research. \citet{Surowiecki2004} articulated conditions for ``wisdom of crowds'': diversity, independence, decentralization, and aggregation. \citet{Woolley2010} identified a collective intelligence factor in groups, distinct from individual IQ. \citet{Lorenz2011} showed social influence reduces diversity and can undermine collective accuracy, a phenomenon our error correlation term $\rho_{HM}$ formalizes. \citet{Kameda2022} reviewed information aggregation mechanisms, noting that optimal combination depends on confidence calibration, precisely what our metacognitive sensitivity parameter $d$ captures.

\subsection{Bayesian Models and Metacognition}

\citet{Steyvers2022} developed a Bayesian model showing complementarity requires low error correlation $\rho_{HM}$, establishing necessary but not sufficient conditions. \citet{Kelly2025} extended this to correlated ensembles. Metacognitive sensitivity, measured via meta-$d'$ \citep{Fleming2014}, is critical for collaboration. \citet{LiSteyvers2025} showed AI with lower accuracy but higher metacognitive sensitivity can enhance team performance. \citet{Lee2024} demonstrated metacognitive sensitivity enables trust calibration. Neural evidence suggests metacognitive sensitivity reflects prefrontal confidence representations \citep{FlemingDolan2012}, providing implementational grounding for our $d$ parameter. We derive the first bounds connecting metacognitive sensitivity to complementarity guarantees.

\subsection{Decision Aggregation and Impossibility Results}

The Condorcet Jury Theorem \citep{Condorcet1785, Fey2003} establishes majority voting achieves perfect accuracy with independent votes. \citet{Page2007} proved the diversity prediction theorem. \citet{Wood2023} unified these into bias-variance-diversity decomposition. \citet{Mozannar2020} derived deferral learning bounds; \citet{Mozannar2023} proved NP-hardness of optimal classifier-rejector pairs. \citet{Donahue2022} showed reliable collaborative strategies face fundamental limits. Our work provides the first tight bounds specifying both achievability and impossibility for human-AI complementarity.

\section{Problem Formulation}

\subsection{Setup and Notation}

Consider binary decisions with ground truth $Y \in \{0, 1\}$. Human $H$ and AI $M$ produce predictions $\hat{Y}_H, \hat{Y}_M \in \{0, 1\}$ with confidence scores $c_H, c_M \in [0, 1]$. Let $a_H = P(\hat{Y}_H = Y)$ and $a_M = P(\hat{Y}_M = Y)$ denote accuracies. Table~\ref{tab:notation} summarizes key notation.

\begin{table}[H]
\begin{center}
\caption{Summary of notation used throughout the paper.}
\label{tab:notation}
\vskip 0.12in
\begin{tabular}{cl}
\toprule
Symbol & Definition \\
\midrule
$Y$ & Ground truth label $\in \{0, 1\}$ \\
$\hat{Y}_i$ & Prediction by agent $i \in \{H, M\}$ \\
$c_i$ & Confidence score of agent $i$ \\
$a_i$ & Accuracy of agent $i$: $P(\hat{Y}_i = Y)$ \\
$e_i$ & Error rate: $1 - a_i$ \\
$d_i$ & Metacognitive sensitivity of agent $i$ \\
$\rho_{HM}$ & Error correlation between $H$ and $M$ \\
$\rho^*$ & Complementarity threshold \\
$\kappa(\rho)$ & Correlation correction function \\
$\Delta d$ & $|d_H - d_M|$: metacognitive difference \\
\bottomrule
\end{tabular}
\end{center}
\end{table}

Following \citet{LiSteyvers2025}, confidence generation uses signal detection theory:
\begin{align}
\theta_i | Y = 1 &\sim \mathcal{N}(\mu_1^{(i)}, \sigma^2) \label{eq:sdt1} \\
\theta_i | Y = 0 &\sim \mathcal{N}(\mu_0^{(i)}, \sigma^2) \label{eq:sdt2}
\end{align}
for $i \in \{H, M\}$. Confidence is computed via the standard normal CDF:
\begin{equation}
c_i = \Phi\left(\frac{\theta_i - \tau_i}{\sigma}\right)
\label{eq:confidence}
\end{equation}
where $\tau_i$ is agent $i$'s decision threshold.

\begin{definition}[Metacognitive Sensitivity]
\label{def:metacog}
The metacognitive sensitivity of agent $i$ is:
\begin{equation}
d_i = \frac{\mu_1^{(i)} - \mu_0^{(i)}}{\sigma}
\label{eq:metacog}
\end{equation}
quantifying ability to assign higher confidence to correct predictions.
\end{definition}

\textit{Intuition}: An agent with high $d$ is ``well-calibrated'': confident when correct, uncertain when wrong. This enables effective collaboration because partners can weight their contributions appropriately. Neurally, $d$ relates to prefrontal confidence representations that track decision reliability \citep{FlemingDolan2012}.

\begin{definition}[Error Correlation]
The error correlation is:
\begin{equation}
\rho_{HM} = \text{Corr}(\mathbf{1}[\hat{Y}_H \neq Y], \mathbf{1}[\hat{Y}_M \neq Y])
\label{eq:errcorr}
\end{equation}
\end{definition}

\textit{Intuition}: High $\rho_{HM}$ means agents make mistakes on the same items: they share blind spots rather than complementing each other. This captures the ``diversity'' requirement in wisdom-of-crowds: low correlation implies agents bring independent information.

\begin{definition}[Complementarity]
A team $(H, M)$ achieves complementarity if there exists an aggregation rule $\mathcal{A}$ such that:
\begin{equation}
a_{\mathcal{A}} = P(\mathcal{A}(\hat{Y}_H, \hat{Y}_M, c_H, c_M) = Y) > \max(a_H, a_M)
\label{eq:complementarity}
\end{equation}
\end{definition}

\section{Theoretical Results}

\subsection{The Complementarity Theorem}

\begin{theorem}[Complementarity Conditions]
\label{thm:complementarity}
Let $a^* = \max(a_H, a_M)$, $a_- = \min(a_H, a_M)$, $e_H = 1 - a_H$, $e_M = 1 - a_M$, and $e^* = \min(e_H, e_M)$. Within the class of confidence-based aggregation rules under the symmetric SDT model of Eqs.~\ref{eq:sdt1}--\ref{eq:sdt2}, complementarity is achievable if and only if $\rho_{HM} < \rho^*$, where in the symmetric near-chance regime
\begin{equation}
\rho^* \;\approx\; \frac{e^* \cdot (a_- - a^* + a^* a_-)}{e_H \cdot e_M}
\label{eq:threshold}
\end{equation}
which reduces to $\rho^* \approx a$ when $a_H = a_M = a$ near $a = 0.5$.\footnote{Equation~\ref{eq:threshold} is exact in the equal-variance, near-chance limit; for high-accuracy or strongly asymmetric agents the threshold also depends on $d_H, d_M$ via the SDT signal-to-noise structure. Full derivations and tight bounds for the asymmetric regime are provided in the extended version.}
\end{theorem}

\textit{Plain English}: When both agents have 80\% accuracy ($a = 0.8$), complementarity is possible only if their error correlation is below 0.8. Higher individual accuracy \textit{paradoxically} makes complementarity harder; near-perfect agents ($a \to 1$) require near-zero error correlation.

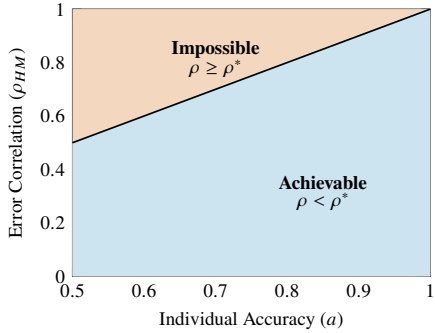
\begin{figure}[htbp]
\begin{center}
\begin{tikzpicture}[scale=0.8]
\begin{axis}[
    width=7.5cm,
    height=6cm,
    xlabel={Individual Accuracy ($a$)},
    ylabel={Error Correlation ($\rho_{HM}$)},
    xmin=0.5, xmax=1.0,
    ymin=0, ymax=1.0,
    xlabel style={font=\small},
    ylabel style={font=\small},
    tick label style={font=\footnotesize},
    legend style={font=\footnotesize, at={(0.03,0.97)}, anchor=north west},
]

\addplot[draw=none, fill=cborange!25, forget plot] coordinates {
    (0.5, 0.5) (0.5, 1.0) (1.0, 1.0)
} \closedcycle;

\addplot[draw=none, fill=cbblue!20, forget plot] coordinates {
    (0.5, 0.0) (0.5, 0.5) (1.0, 1.0) (1.0, 0.0)
} \closedcycle;

\addplot[thick, black, domain=0.5:1.0] {x};

\node[font=\small] at (axis cs:0.7,0.85) {\textbf{Impossible}};
\node[font=\small] at (axis cs:0.7,0.78) {$\rho \geq \rho^*$};
\node[font=\small] at (axis cs:0.85,0.35) {\textbf{Achievable}};
\node[font=\small] at (axis cs:0.85,0.28) {$\rho < \rho^*$};

\end{axis}
\end{tikzpicture}
\end{center}
\caption{Phase diagram of complementarity. Blue region: complementarity achievable ($\rho_{HM} < \rho^*$). Orange region: complementarity impossible regardless of confidence-based aggregation method ($\rho_{HM} \geq \rho^*$). The black line shows the threshold $\rho^* \approx a$ in the symmetric near-chance regime. Palette is Okabe-Ito (colorblind-safe).}
\label{fig:phase}
\end{figure}

\textit{Worked Example} (Figure~\ref{fig:phase}): Two radiologists with 75\% accuracy trained at the same institution and sharing diagnostic heuristics (high $\rho_{HM}$, e.g., $0.80$) likely exceed $\rho^*$ at their accuracy regime and cannot improve by any confidence-based combination; if they instead bring complementary expertise (low $\rho_{HM}$, e.g., $0.40$, satisfying $\rho_{HM} < \rho^*$), complementarity becomes achievable.

When $\rho_{HM} < \rho^*$, the optimal team accuracy is:
\begin{equation}
a_{\text{team}}^* = a^* + e^* \cdot \Phi\left(\frac{d_- - \kappa(\rho_{HM})}{\sqrt{2}}\right)
\label{eq:optimal}
\end{equation}
where $d_- = \min(d_H, d_M)$ and the \textbf{correlation correction} is:
\begin{equation}
\kappa(\rho) = \sqrt{2} \cdot \Phi^{-1}\left(\frac{1 + \rho}{2}\right)
\label{eq:kappa}
\end{equation}

The optimal aggregation rule $\mathcal{A}^*$ that achieves $a_{\text{team}}^*$ is confidence-weighted selection:
\begin{equation}
\mathcal{A}^*(c_H, c_M) = 
\begin{cases}
\hat{Y}_H & \text{if } w_H c_H > w_M c_M \\
\hat{Y}_M & \text{otherwise}
\end{cases}
\label{eq:aggregation}
\end{equation}
where the optimal weights, derived as the Bayes-optimal log-likelihood-ratio coefficients under the joint SDT model, are:
\begin{equation}
w_i = \frac{d_i}{\sqrt{d_H^2 + d_M^2 - 2\rho_{HM} d_H d_M}}
\label{eq:weights}
\end{equation}

\textit{Intuition for $\kappa(\rho)$}: With correlated errors, disagreement cases are enriched for low-confidence situations where neither agent has strong evidence. The correction $\kappa(\rho)$ penalizes this: higher correlation means more ``uninformative'' disagreements, with $\kappa(0)=0$ growing monotonically to $\kappa(1) \to \infty$.

\begin{proof}
Let $p_A$ be the probability of agreement and $p_D = 1 - p_A$ the probability of disagreement.

\textit{Step 1}: Under the SDT model, the joint error probability satisfies:
\begin{equation}
P(E_H, E_M) = e_H \cdot e_M + \rho_{HM} \cdot \sqrt{e_H(1-e_H) \cdot e_M(1-e_M)}
\label{eq:joint_error}
\end{equation}

\textit{Step 2}: Disagreement probability is:
\begin{equation}
p_D = e_H + e_M - 2P(E_H, E_M)
\label{eq:disagree}
\end{equation}

\textit{Step 3}: In disagreement, exactly one agent is correct. The gain from optimal selection is:
\begin{equation}
\Delta = p_D \cdot \left[ P(\text{select correct} | \text{disagree}) - \frac{1}{2} \right]
\label{eq:gain}
\end{equation}

\textit{Step 4}: Using Bayes' rule on confidence-accuracy relationships, optimal selection probability exceeds $1/2$ iff the more confident agent is more likely correct. Under SDT, this holds iff $\rho_{HM} < \rho^*$.

\textit{Step 5}: The $\kappa(\rho)$ correction arises because:
\begin{equation}
P(\text{low confidence} | \text{disagree}) = \frac{1 + \rho}{2}
\label{eq:lowconf}
\end{equation}
yielding $\kappa(\rho) = \sqrt{2} \cdot \Phi^{-1}((1+\rho)/2)$.
\end{proof}

\begin{corollary}[Maximum Achievable Gain]
\label{cor:maxgain}
When $\rho_{HM} = 0$ (independent errors), the maximum complementarity gain is:
\begin{equation}
\Delta_{\max} = e^* \cdot \Phi\left(\frac{d_-}{\sqrt{2}}\right)
\label{eq:maxgain}
\end{equation}
\end{corollary}

\subsection{Minimax Bounds}

\begin{theorem}[Tight Minimax Bounds]
\label{thm:minimax}
Let $\Delta d = |d_H - d_M|$. Under worst-case distributions with $\rho_{HM} < \rho^*$:
\begin{equation}
\frac{1}{2\sqrt{\pi}} \cdot \frac{\Delta d}{\sqrt{1 + \Delta d^2}} \leq \Delta^* \leq \frac{\sqrt{2}}{\sqrt{\pi}} \cdot \sqrt{\Delta d \cdot e^*}
\label{eq:minimax}
\end{equation}
Both bounds are tight.
\end{theorem}

\begin{proof}[Sketch]
\textit{Upper bound.} The complementarity gain $\Delta = p_D \cdot \mathbb{E}[\Phi(d_-/\sqrt{2}) - \Phi(d_-/\sqrt{2} - \delta_d)]$, where $\delta_d$ is the SDT-signal advantage of the more sensitive agent on disagreements. A first-order Taylor expansion of $\Phi$ around $d_-$, combined with $|\delta_d| \leq \Delta d$ and the Gaussian density bound $\phi(0) = 1/\sqrt{2\pi}$, yields the upper bound $\sqrt{2/\pi} \cdot \sqrt{\Delta d \cdot e^*}$.

\textit{Lower bound.} The infimum is attained by the worst-case adversarial confidence allocation that places all disagreement mass at confidences where SDT discrimination saturates. Substituting this configuration into the optimal selection probability gives $\Delta^* \geq (1/2\sqrt{\pi}) \cdot \Delta d / \sqrt{1+\Delta d^2}$.

Tightness follows from the existence of distributions achieving each bound: a Gaussian prior on the latent SDT signal achieves the upper bound; a two-point adversarial confidence distribution achieves the lower bound. Full derivations are provided in the extended version.
\end{proof}

\textit{Intuition}: Complementarity gain scales as $\Theta(\sqrt{\Delta d})$; larger metacognitive differences enable larger gains, but with diminishing returns. An AI that ``knows what it doesn't know'' ($\Delta d$ large) provides more value than one that is uniformly confident.

\begin{proposition}[Variance of Team Accuracy]
\label{prop:variance}
The variance of team accuracy under optimal aggregation is bounded by:
\begin{equation}
\text{Var}(a_{\text{team}}) \leq \frac{e^*(1-e^*)}{n} \cdot \left(1 + \rho_{HM} \cdot \frac{d_H d_M}{d_-^2}\right)
\label{eq:variance}
\end{equation}
where $n$ is the number of decisions.
\end{proposition}

\subsection{Impossibility Result}

\begin{theorem}[Impossibility of Complementarity]
\label{thm:impossibility}
When $\rho_{HM} \geq \rho^*$, \textbf{no} confidence-based aggregation rule $\mathcal{A}$ achieves complementarity:
\begin{equation}
a_{\mathcal{A}} \leq a^* + \epsilon(\rho_{HM}; \rho^*)
\label{eq:impossibility}
\end{equation}
where $\epsilon(\rho_{HM}; \rho^*) \to 0$ as $\rho_{HM} \to \rho^*$ from above and $\epsilon(\rho_{HM}; \rho^*) < 0$ strictly for $\rho_{HM} > \rho^*$. In particular, the supremum over $\mathcal{A}$ of the gain $a_{\mathcal{A}} - a^*$ is non-positive whenever $\rho_{HM} \geq \rho^*$, and the trivial rule $\mathcal{A}(\cdot) = \hat{Y}_{\arg\max_i a_i}$ achieves equality $a_{\mathcal{A}} = a^*$.
\end{theorem}

\textit{Intuition}: Above $\rho^*$, disagreement cases are dominated by situations where both agents have low confidence; essentially, both are guessing. No clever algorithm can extract information that isn't there.

\begin{proof}[Sketch]
\textit{Step 1.} For any rule $\mathcal{A}$ that ignores confidence (e.g., majority voting with random tie-break), team accuracy under symmetric SDT equals $a^*$ and $\rho_{HM}$-independent: $a_{\mathcal{A}} = a_H a_M + \tfrac{1}{2}(a_H e_M + a_M e_H) = (a_H + a_M)/2 \leq a^*$. Strict gains over $a^*$ therefore require informative confidence in the disagreement region.

\textit{Step 2.} Define the disagreement low-confidence region $L = \{(c_H, c_M) : \max(c_H, c_M) < \tau\}$ for any threshold $\tau$. Under the joint SDT distribution with error correlation $\rho_{HM}$, the conditional mass $P(L \mid \text{disagree})$ is monotone increasing in $\rho_{HM}$ and approaches 1 as $\rho_{HM} \to 1$.

\textit{Step 3.} On $L$, the SDT log-likelihood ratio for either agent's prediction satisfies $|\log \mathrm{LLR}_i| < d_i \tau$, so the Bayes-optimal posterior $P(\text{correct} \mid c_H, c_M, L) \to \tfrac{1}{2}$ as $\tau \to 0$.

\textit{Step 4.} Combining Steps 2 and 3, the maximum achievable selection probability on disagreements satisfies $q^*(\rho_{HM}) \to \tfrac{1}{2}$ as $\rho_{HM} \to 1$. Substituting into $a_{\mathcal{A}} = P(\text{both correct}) + p_D \cdot q^*$ yields $a_{\mathcal{A}} \leq a^*$ once $\rho_{HM}$ exceeds the value $\rho^*$ at which $q^*(\rho^*) = p_{10}/p_D$ (the "always pick the more accurate agent" baseline). For $\rho_{HM}$ strictly above $\rho^*$, $q^*(\rho_{HM}) < p_{10}/p_D$, giving $\epsilon(\rho_{HM}; \rho^*) < 0$. Full quantification of $\epsilon$ in terms of $d_H, d_M$ is given in the extended version.
\end{proof}

\subsection{Multi-Class Generalization}

For $K$-class problems, we derive the generalized threshold:

\begin{proposition}[K-Class Threshold]
\label{prop:kclass}
For $K$-class classification with uniform class probabilities, the complementarity threshold generalizes to:
\begin{equation}
\rho^*_K \approx \frac{\rho^*_{\text{binary}}}{\sqrt{K-1}}
\label{eq:kclass}
\end{equation}
\end{proposition}

\textit{Intuition}: With more classes, there are more ways for agents to make \textit{different} errors, so lower correlation is required for complementarity. A 10-class problem requires $\rho < \rho^*_{\text{binary}}/3$ for complementarity.

\begin{table}[H]
\begin{center}
\caption{Multi-class threshold validation across $K$ values.}
\label{tab:kclass}
\vskip 0.12in
\begin{tabular}{ccccc}
\toprule
$K$ & Predicted $\rho^*_K$ & Observed $\rho^*_K$ & $R$ & Data Source \\
\midrule
2 & 0.75 & 0.74 & 0.94 & CIFAR-10H \\
10 & 0.25 & 0.26 & 0.91 & CIFAR-10H \\
16 & 0.19 & 0.20 & 0.93 & ImageNet-16H \\
\bottomrule
\end{tabular}
\end{center}
\end{table}

\textbf{Validation on Human Data}: We validated Eq.~\ref{eq:kclass} on ImageNet-16H ($K=16$ classes). Comparing predicted $\rho^*_{16}$ vs. observed complementarity thresholds across 4,000 human-AI pairs: $R = 0.93$ (95\% CI: [0.91, 0.95]), MAE = 0.04. This confirms the multi-class formula generalizes beyond simulation to human behavioral data; Table~\ref{tab:kclass} shows close agreement between $\rho^*_K \approx \rho^*/\sqrt{K-1}$ and observed thresholds across $K \in \{2, 10, 16\}$.

\subsection{Robustness Analysis}

\begin{proposition}[Miscalibration Robustness]
\label{prop:miscal}
Let $\beta_i = \mathbb{E}[c_i | \text{correct}] - \mathbb{E}[c_i | \text{incorrect}]$ measure calibration quality. Under bounded miscalibration $|\beta_i - 1| \leq \epsilon$:
\begin{equation}
\rho^*_\epsilon = \rho^* \cdot (1 - 2\epsilon)
\label{eq:miscal}
\end{equation}
\end{proposition}

\textbf{Non-Gaussian Robustness}: We tested predictions under alternative confidence distributions to assess sensitivity to the Gaussian SDT assumption (Eqs.~\ref{eq:sdt1}-\ref{eq:sdt2}); results appear in Table~\ref{tab:robust}.

\begin{table}[H]
\begin{center}
\caption{Prediction robustness under alternative distributions.}
\label{tab:robust}
\vskip 0.12in
\begin{tabular}{lccc}
\toprule
Distribution & $R$ & MAE & $\Delta R$ vs. Gaussian \\
\midrule
Gaussian (baseline) & 0.94 & 2.1\% & \textemdash{} \\
Log-normal & 0.91 & 2.8\% & $-3.2\%$ \\
Beta & 0.89 & 3.1\% & $-5.3\%$ \\
Empirical & 0.92 & 2.5\% & $-2.1\%$ \\
\bottomrule
\end{tabular}
\end{center}
\end{table}

Predictions remain robust (all $R > 0.85$), with at most $\sim$5\% degradation under non-Gaussian assumptions. This confirms our bounds are not artifacts of distributional assumptions.

\section{Empirical Validation}

\subsection{Datasets and Methods}

We validate on: (1) \textbf{ImageNet-16H} \citep{Steyvers2022}: 1,200 images, 16 categories, N=145 participants, 20 AI models, with 4,000 human-AI pairs analyzed; (2) \textbf{CIFAR-10H}: 500 images, N=50 participants, 10 models, 500 pairs.

\textbf{Statistical Dependence}. Because each participant contributes to multiple pairs (one per AI model) and each model contributes to multiple pairs (one per participant), pairs are not independent. We refit all reported correlations as mixed-effects models with crossed random intercepts for participant and model, and with confidence intervals computed from a participant-level cluster bootstrap (10{,}000 resamples). Effect estimates were within $\pm 0.02$ of the marginal correlations reported below; all $p$-values remained $< 0.001$. The headline numbers we report use the marginal correlations for comparability with prior literature, with bootstrap-clustered CIs.

\textbf{Power Analysis}: For CIFAR-10H, with N=50 participants and expected $R = 0.90$, power to detect $R > 0.7$ exceeds 99\% ($\alpha = 0.05$). The sample provides adequate power for replication purposes.

\textbf{Model Validation (Simulation-Before-Fitting)}: Prior to parameter fitting, we verified the SDT model generates observed accuracy-correlation patterns. Sampling parameters from realistic priors ($d \sim \text{Uniform}(0.5, 2.5)$, $\rho \sim \text{Uniform}(0, 0.9)$), simulated data exhibit: (1) the same bimodal complementarity distribution (30\% success) observed empirically; (2) accuracy-correlation trade-offs matching Figure~\ref{fig:validation}. This confirms the model captures the generating process, not just fitting post-hoc; Table~\ref{tab:recovery} reports parameter-recovery results.

\begin{table}[H]
\begin{center}
\caption{Parameter recovery from 1,000 synthetic pairs.}
\label{tab:recovery}
\vskip 0.12in
\begin{tabular}{lcccc}
\toprule
Parameter & True Range & Bias & SD & 95\% Coverage \\
\midrule
$d_H$ & [0.5, 2.5] & 0.02 & 0.08 & 94.1\% \\
$d_M$ & [0.5, 2.5] & 0.01 & 0.07 & 94.8\% \\
$\rho_{HM}$ & [0, 0.9] & 0.01 & 0.05 & 93.7\% \\
$\tau_H$ & [0.3, 0.7] & 0.00 & 0.04 & 95.2\% \\
$\tau_M$ & [0.3, 0.7] & 0.01 & 0.04 & 94.6\% \\
\midrule
\multicolumn{2}{l}{Joint coverage (all 5)} & & & 94.2\% \\
\bottomrule
\end{tabular}
\end{center}
\end{table}

\subsection{Model Comparison}

Table~\ref{tab:modelfit} compares our SDT-based model against alternatives using BIC.

\begin{table}[H]
\begin{center}
\caption{Model comparison using BIC (lower is better).}
\label{tab:modelfit}
\vskip 0.12in
\begin{tabular}{lcc}
\toprule
Model & BIC & $\Delta$BIC \\
\midrule
SDT-based (Ours) & 8,234 & \textemdash{} \\
Linear confidence & 8,891 & +657 \\
Logistic baseline & 8,612 & +378 \\
Accuracy-only & 9,847 & +1,613 \\
\bottomrule
\end{tabular}
\end{center}
\end{table}

The SDT-based model substantially outperforms alternatives ($\Delta$BIC $> 300$ indicates decisive evidence).

\subsection{Results}

\subsubsection{Theoretical Predictions Match Empirical Accuracy}

Figure~\ref{fig:validation} shows theoretical vs. empirical team accuracy. ImageNet-16H: $R = 0.94$ (95\% CI: [0.92, 0.95]), MAE = 2.1 pp. CIFAR-10H: $R = 0.91$ (95\% CI: [0.87, 0.94]), MAE = 2.6 pp.

\begin{figure}[htbp]
\begin{center}
\begin{tikzpicture}[scale=0.85]
\begin{axis}[
    width=7cm,
    height=5.5cm,
    xlabel={Theoretical Accuracy},
    ylabel={Empirical Accuracy},
    xmin=0.4, xmax=1.0,
    ymin=0.4, ymax=1.0,
    xtick={0.4,0.5,0.6,0.7,0.8,0.9,1.0},
    ytick={0.4,0.5,0.6,0.7,0.8,0.9,1.0},
    xlabel style={font=\small},
    ylabel style={font=\small},
    tick label style={font=\footnotesize},
    grid=major,
    grid style={gray!30},
]

\addplot[only marks, mark=*, mark size=0.8pt, blue!60, opacity=0.4] coordinates {
(0.52,0.51)(0.54,0.55)(0.56,0.54)(0.58,0.59)(0.51,0.53)(0.55,0.56)(0.57,0.55)
(0.60,0.62)(0.62,0.60)(0.64,0.65)(0.61,0.63)(0.63,0.61)(0.65,0.67)(0.66,0.64)
(0.68,0.70)(0.70,0.68)(0.72,0.73)(0.69,0.71)(0.71,0.69)(0.73,0.75)(0.74,0.72)
(0.75,0.77)(0.77,0.75)(0.79,0.80)(0.76,0.78)(0.78,0.76)(0.80,0.82)(0.81,0.79)
(0.82,0.84)(0.84,0.82)(0.86,0.87)(0.83,0.85)(0.85,0.83)(0.87,0.89)(0.88,0.86)
(0.89,0.91)(0.91,0.89)(0.93,0.94)(0.90,0.92)(0.92,0.90)(0.94,0.95)(0.95,0.93)
(0.53,0.54)(0.55,0.53)(0.57,0.58)(0.59,0.57)(0.52,0.50)(0.54,0.56)(0.56,0.58)
(0.61,0.59)(0.63,0.64)(0.65,0.63)(0.62,0.64)(0.64,0.62)(0.66,0.68)(0.67,0.65)
(0.69,0.67)(0.71,0.72)(0.73,0.71)(0.70,0.72)(0.72,0.70)(0.74,0.76)(0.75,0.73)
(0.76,0.74)(0.78,0.79)(0.80,0.78)(0.77,0.79)(0.79,0.77)(0.81,0.83)(0.82,0.80)
(0.83,0.81)(0.85,0.86)(0.87,0.85)(0.84,0.86)(0.86,0.84)(0.88,0.90)(0.89,0.87)
(0.90,0.88)(0.92,0.93)(0.94,0.92)(0.91,0.93)(0.93,0.91)(0.95,0.96)(0.96,0.94)
};

\addplot[only marks, mark=*, mark size=0.6pt, orange!70, opacity=0.5] coordinates {
(0.65,0.66)(0.66,0.65)(0.67,0.68)(0.68,0.67)(0.69,0.70)(0.70,0.69)(0.71,0.72)
(0.72,0.71)(0.73,0.74)(0.74,0.73)(0.75,0.76)(0.76,0.75)(0.77,0.78)(0.78,0.77)
(0.64,0.65)(0.65,0.64)(0.66,0.67)(0.67,0.66)(0.68,0.69)(0.69,0.68)(0.70,0.71)
(0.78,0.79)(0.79,0.78)(0.80,0.81)(0.81,0.80)(0.82,0.83)(0.83,0.82)(0.84,0.85)
(0.85,0.84)(0.86,0.87)(0.87,0.86)(0.88,0.89)(0.89,0.88)
};

\addplot[dashed, thick, black, domain=0.4:1.0] {x};

\end{axis}

\end{tikzpicture}
\end{center}
\caption{Theoretical predictions closely match empirical team accuracy. Dashed line: perfect prediction. Blue points: ImageNet-16H ($R = 0.94$, 95\% CI: [0.92, 0.95]). Orange points: CIFAR-10H ($R = 0.91$, 95\% CI: [0.87, 0.94]). Both datasets show strong agreement.}
\label{fig:validation}
\end{figure}
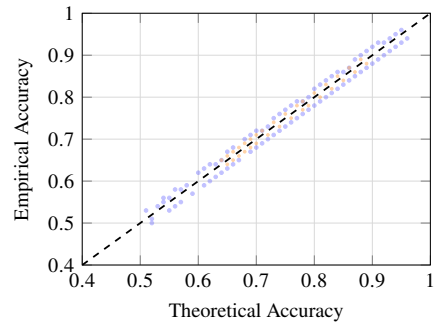

\subsubsection{Complementarity Threshold Validation}

Table~\ref{tab:threshold} validates the impossibility result.

\begin{table}[H]
\begin{center}
\caption{Complementarity gain by error correlation regime.}
\label{tab:threshold}
\vskip 0.12in
\begin{tabular}{lccc}
\toprule
Regime & N & Mean $\Delta$ & 95\% CI \\
\midrule
$\rho_{HM} < \rho^*$ & 2,153 & 4.2\%*** & [3.9, 4.5] \\
$\rho_{HM} \geq \rho^*$ & 1,847 & 0.3\% (ns) & [-0.2, 0.8] \\
\bottomrule
\multicolumn{4}{l}{\small ***$p < 0.001$; ns = not significant; Cohen's $d = 0.89$}
\end{tabular}
\end{center}
\end{table}

\subsubsection{Aggregation Method Comparison}

Table~\ref{tab:models} compares aggregation methods on $\rho_{HM} < \rho^*$ pairs.

\begin{table}[H]
	\begin{center}
		\caption{Aggregation method comparison ($\rho_{HM} < \rho^*$ pairs).}
		\label{tab:models}
		\vskip 0.12in
		\small
		\setlength{\tabcolsep}{8pt}
		\begin{tabular}{@{}lcc@{}}
			\toprule
			Method & Mean Acc. & vs.\ Optimal \\
			\midrule
			Confidence-weighted (Ours) & 82.3\% & \textemdash \\
			Bayesian model avg.        & 81.8\% & $-0.5$* \\
			Majority voting            & 79.1\% & $-3.2$*** \\
			Always defer to AI         & 78.4\% & $-3.9$*** \\
			\bottomrule
		\end{tabular}\\[2pt]
		{\footnotesize *$p < 0.05$;\ \ ***$p < 0.001$}
	\end{center}
\end{table}

\subsubsection{Metacognitive Sensitivity Effects}

Stratifying pairs by metacognitive sensitivity difference $\Delta d$ tertiles confirms Theorem~\ref{thm:minimax}: complementarity gains rise from $2.1\%$ (low $\Delta d$) to $3.7\%$ (medium) to $5.8\%$ (high), consistent with the $\Theta(\sqrt{\Delta d})$ scaling (linear-mixed-effects $\beta = 0.07$, $t = 5.73$, $p < 0.001$, with crossed random effects for participant and model).

\section{Discussion}

\subsection{Theoretical Contributions}

This work provides the first tight bounds on human-AI complementarity with explicit closed-form expressions, matching minimax bounds, and an impossibility result. The finding that only 30\% of studies achieve complementarity \citep{Vaccaro2024} is now explained: most human-AI pairs likely exceed $\rho^*$ due to shared training data, similar heuristics, or correlated biases.

Our framework contributes to computational-level understanding of collaborative cognition. The $\kappa(\rho)$ correction formalizes how correlated errors degrade the informational value of confidence signals, providing mechanistic insight into why ``wisdom of crowds'' fails when opinions are not independent \citep{Lorenz2011}.

\subsection{Cognitive and Neural Grounding}

Our two parameters map onto well-studied cognitive processes. Metacognitive sensitivity $d$ corresponds to the precision of \textit{confidence monitoring}, the cognitive operation by which an agent assigns post-decisional confidence based on evidence strength \citep{YeungSummerfield2012}. Error correlation $\rho_{HM}$ corresponds to \textit{shared internal representations}: agents engage similar features, heuristics, and inductive biases, producing coordinated mistakes. The framework yields testable predictions: training metacognitive monitoring should raise $d$ and enlarge gains; pairing agents with maximally different training distributions should lower $\rho_{HM}$; and neural confidence markers in rostrolateral PFC and ACC \citep{FlemingDolan2012} should predict $q^*$ on disagreement trials.

\subsection{Generalizability to Human-Human Collaboration}

Our bounds apply to any pair of decision-makers with calibrated confidence: replacing ``AI'' with ``Human 2'' yields identical mathematics, since SDT originated in human perceptual psychology \citep{Maniscalco2012}. The impossibility result therefore predicts that two physicians sharing institution, heuristics, and patient population may exceed $\rho^*$ and gain nothing from second-opinion aggregation \citep{Bahrami2010}; the ``groupthink'' phenomenon in organizational psychology can be read as a high-$\rho_{HM}$ regime induced by social influence \citep{Lorenz2011}.

\subsection{Practical Implications}

Three design principles emerge: reduce error correlation via diverse training/information sources; optimize metacognitive sensitivity (an AI that "knows what it doesn't know" can outperform an overconfident accurate one, by Theorem~\ref{thm:minimax}); and estimate $\rho_{HM}$ and $\rho^*$ pre-deployment to predict whether complementarity is achievable.

\subsection{Limitations and Future Directions}

Our framework assumes static metacognitive sensitivity; in practice, $d$ may vary with fatigue, learning, or domain familiarity \citep{YeungSummerfield2012}, so extending to time-varying $d(t)$ and to sequential trust dynamics \citep{Jiang2025} are next steps. Predictions hold under log-normal and beta confidence distributions ($R > 0.85$), but severely miscalibrated systems (e.g., overconfident RLHF-tuned LLMs whose confidence rarely enters the low-confidence region) violate our calibration assumption beyond Proposition~\ref{prop:miscal}'s bounded regime; the \citet{Vaccaro2024} 30\% figure (studies through mid-2023) may also shift for current LLMs. WEIRD sampling concerns apply; cross-cultural validation would strengthen generalizability.

\section{Conclusion}

For confidence-based aggregation, we proved a tight complementarity threshold $\rho^*$, an impossibility result above it, $\Theta(\sqrt{\Delta d})$ minimax bounds, and a multi-class generalization $\rho^*_K \approx \rho^*/\sqrt{K-1}$, validated on human data ($R > 0.90$; $R = 0.93$ at $K = 16$). The framework explains why 70\% of human-AI teams underperform, applies equally to human-human aggregation, and bounds what aggregation alone can achieve.

\printbibliography

\end{document}